\documentclass[conference]{IEEEtran}
\IEEEoverridecommandlockouts
\usepackage{cite,url}
\usepackage{amsmath,amssymb,amsfonts,amsthm}
\usepackage{graphicx}
\usepackage{textcomp}
\usepackage{xcolor}
\def\BibTeX{{\rm B\kern-.05em{\sc i\kern-.025em b}\kern-.08em
    T\kern-.1667em\lower.7ex\hbox{E}\kern-.125emX}}

\usepackage{bm}
\usepackage{lipsum,booktabs}
\usepackage{algorithm,algorithmicx,algpseudocode}
\DeclareMathSymbol{\shortminus}{\mathbin}{AMSa}{"39}
\newcommand{\B}{\hat{B}}
\renewcommand{\P}{\mathbb{P}}

\newtheorem{theorem}{Theorem}
\newtheorem{proposition}{Proposition}
\newtheorem{definition}{Definition}
\renewcommand{\P}{\mathbb{P}}

\makeatletter
\newcommand\ntiny{\@setfontsize\ntiny\@vipt\@viipt}
\makeatother

\begin{document}

\title{Constraining Anomaly Detection with Anomaly-Free Regions}




\author{\IEEEauthorblockN{Maximilian Toller}
\IEEEauthorblockA{\textit{Know Center} \\
Graz, Austria \\
mtoller@know-center.at}
\and
\IEEEauthorblockN{Hussain Hussain}
\IEEEauthorblockA{\textit{Graz University of Technology} \\
Graz, Austria \\
hussain@tugraz.at}
\and
\IEEEauthorblockN{Roman Kern}
\IEEEauthorblockA{\textit{Graz University of Technology} \\
Graz, Austria \\
rkern@tugraz.at}
\and
\IEEEauthorblockN{Bernhard C. Geiger}
\IEEEauthorblockA{\textit{Know Center} \\
Graz, Austria \\
geiger@ieee.org}
\thanks{The Know-Center is funded within COMET Competence Centers for Excellent Technologies under the auspices of the Austrian Federal Ministry of Transport, Innovation and Technology, the Austrian Federal Ministry for Digital and of Economic Affairs, the Austrian Research Promotion Agency (FFG), and by the province of Styria. The COMET programme is managed by the FFG.}

}

\maketitle

\begin{abstract}
        We propose the novel concept of anomaly-free regions (AFR) to improve anomaly detection.
		An AFR is a region in the data space for which it is known that there are no anomalies inside it, e.g., via domain knowledge.
        This region can contain any number of normal data points and can be anywhere in the data space.
        AFRs have the key advantage that they constrain the estimation of the distribution of non-anomalies: The estimated probability mass inside the AFR must be consistent with the number of normal data points inside the AFR.
        Based on this insight, we provide a solid theoretical foundation and a reference implementation of anomaly detection using AFRs.
        Our empirical results confirm that anomaly detection constrained via AFRs improves upon unconstrained anomaly detection.
		Specifically, we show that, when equipped with an estimated AFR, an efficient algorithm based on random guessing becomes a strong baseline that several widely-used methods struggle to overcome.
        On a dataset with a ground-truth AFR available, the current state of the art is outperformed.
\end{abstract}


\section{Introduction}

Anomaly detection (AD) is a common data mining task where the goal is to distinguish deviating data, i.e., anomalies, from normal data.
Many common AD methods are based on the so called concentration assumption~\cite{ruff2021unifying}. 
The concentration assumption implies that normal data points form dense clusters, where normal data points have low distances to other normal data points.
Anomalies are then typically found in less dense regions, or have a higher distance to other points.
The concentration assumption is made by a number of prominent AD methods (e.g. LOF~\cite{breunig2000lof}, SVDD~\cite{tax2004support}, Isolation Forest~\cite{liu2008isolation}, and others~\cite{breunig2000lof,tax2004support,liu2008isolation,toller2021cluster,zhang2021elite,pang2019deep,jiang2023anomaly,perini2023learning}) and has demonstrated its usefulness in a number of scenarios of practical relevance.
However, in cases where additional knowledge about the nature of the dataset is available, such knowledge cannot be effectively integrated in the assumption.

\paragraph{Concept.} 
We propose a novel concept, namely Anomaly Free Regions (AFR), which is  similar to the concentration assumption.
An AFR is a region in the data space for which it is known that there are no anomalies inside, e.g., via domain knowledge.
This region can contain any number of normal data points and can be anywhere in the data space.
The AFR does not need to coincide with a high-density region---it could be a region that contains no data points at all.
Hence, using AFRs is appropriate in cases where domain knowledge allows one to derive such regions and certain parts of the data space do not exhibit anomalous behavior.
Prior works that include such domain knowledge are related to mechanical engineering~\cite{wei2019review,sha2022regional}, mobile data traffic monitoring~\cite{ahmed2022rcad}, building energy consumption~\cite{himeur2021artificial}, and digital twins~\cite{huang2021digital}.

\paragraph{Novelty.} 
Several typical concepts in AD literature are comparable to AFRs, such as density level sets~\cite{ruff2021unifying} and inlier clusters~\cite{degirmenci2022efficient}.
AFRs have two distinct properties: a) they are not required to contain normal data points; b) they are a subset of the data space that has zero probability of containing anomalies.
As such, AFRs encode prior information about the anomaly class, which makes them comparable to semi-supervised AD methods (e.g. \cite{pang2023deep,chang2023data}) and Positive-Unlabeled (PU) learning~\cite{ju2020pumad,perini2023learning}.
However, in contrast to semi-supervised methods, AFRs do not require explicit labels for either normal or abnormal data points.

\paragraph{Illustration.} 
AFRs are especially useful for probabilistic AD methods since they constrain the estimation of the distribution of non-anomalies: The estimated probability mass inside the AFR must be consistent with the number of normal data points inside the AFR.
For example, if the AFR is empty, we would expect very little probability mass inside it.
Conversely, if the AFR contains many non-anomalies, we would expect the distribution of non-anomalies to have much of its probability mass inside the AFR.
Figure~\ref{fig:figure1} provides an illustration of this benefit.

\begin{figure*}
\centering
\includegraphics[scale=0.275]{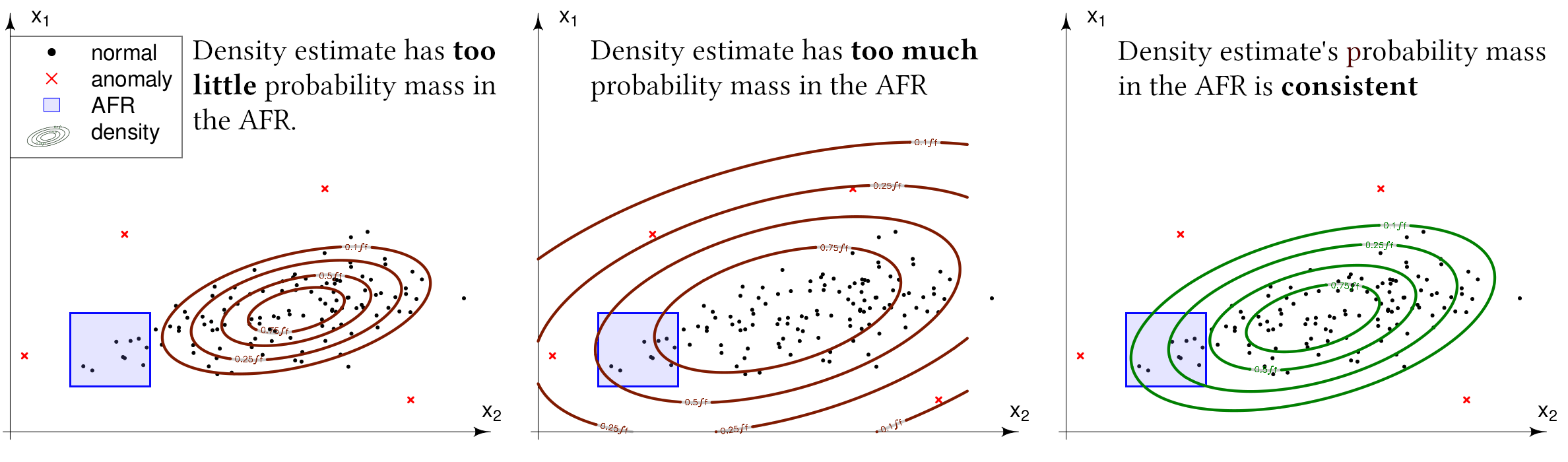}
\caption{Constraining density estimation with an AFR. For a fixed toy dataset and a fixed AFR, three different density estimates for normal data points are compared. The \textit{left} and \textit{middle} density estimates are inconsistent and give a too low likelihood of observing the (known normal) data points inside the AFR. The \textit{right} density estimate makes use of the available knowledge (AFR) and gives a reasonable likelihood for the data points inside the AFR.}
\label{fig:figure1}
\end{figure*}

\paragraph{Mathematical foundation.}
Due to their mathematical nature, we can analyse AFRs in great detail.
Thus, we build a solid foundation for a novel, \textit{constrained} form of AD and construct a concrete reference algorithm.
This algorithm provides a sound basis for further methods based on AFRs.

\paragraph{Real-world usefulness.}
In settings where no AFR is known, an AFR can be estimated.
We show the proposed algorithm behaves solidly in cases where no AFR is known and is competitive with state of the art AD methods.

\paragraph{Contributions.}
Our contributions include the following:
\begin{itemize}
    \item Proposal of the novel AFR concept
    \item Theoretical foundation for constrained anomaly detection
    \item Proposition of an algorithm for constrained density estimation  and application to anomaly detection
    \item Demonstration of the behaviour on several benchmark datasets where no AFRs are known or may not exist.
    \item Demonstration of very good performance on a benchmark dataset where an AFR is known.
\end{itemize}

\section{Related Work}
There are two main branches of literature that describe theory and methods which relate to constrained AD via AFRs.
\begin{enumerate}
	\item Identification of two-component mixture models where one component is partially known and the other is unknown.
	\item Semi-supervised AD, one-class classification and PU learning when positive class labels represents anomalies.
\end{enumerate}
The first branch stems from computational statistics and information theory.
It is primarily concerned with distributions/mixtures of the form 
\begin{equation}\label{eq:start}
	g(x) = (1-p)f_N(x,\bm{\theta}) + p f_\infty(x)
\end{equation} 
where $f_N$ is the ``normal'' distribution, $\bm{\theta}$ are its parameters, $f_\infty$ is the unknown (anomaly) distribution, and $p$ is the mixing parameter.
Works within this branch typically start their analysis with (a variant of) Equation~\eqref{eq:start} and then split according to the assumptions they make.
These distinguishing assumptions are \textit{symmetry of} $f_\infty$ 
~\cite{ma2015flexible}, \textit{membership of $f_N$ in the Gaussian family} and \textit{linearly constrained $f_\infty$}~\cite{al2017semiparametric}, \textit{Gaussian zero-mean $f_N$}~\cite{song2010estimating}, and that the \textit{true form of $f_N$ and its parameters} are given~\cite{shen2018mm,shokirov2010on,patra2016estimation}.
These works are related to our work because we also start our theoretical investigation with Equation~\eqref{eq:start}.
Our work is different from the cited works since we constrain $f_\infty$ via AFRs.
As we will see later, this implies a consistency constraint via the Law of Large Numbers (LLN) that further distinguishes our setup.
In the context of AD, AFRs are easier to justify than the assumptions made in related work, e.g., symmetry of $f_\infty$.

The second branch of related work stems from machine learning and data science.
Semi-supervised AD, and PU learning are a form of classification that uses labeled and unlabeled data to delimit a region where most data from the negative class can be found.
In the context of AD, this translates to learning to delimit the region where normal data live using labeled anomalies and unlabeled data.
Recently, many relevant semi-supervised and PU-learning schemes were proposed for AD.
These are \textit{deep}~\cite{pang2023deep}, use \textit{fast random projections}~\cite{bhattacharya2021fast}, use \textit{metric-learning}~\cite{ju2020pumad}, and are \textit{specialized on tabular data}~\cite{chang2023data}.
Many other related approaches are summarized in the survey by Ruff et al.~\cite{ruff2021unifying}.
Semi-supervised AD methods are related to our work since they also leverage available prior information to facilitate the detection task.
However, unlike these methods, our work does not focus on generalizing from labels and instead uses an AFR to encode prior information about the anomaly class.

\section{Theory}\label{sec:theory}
The starting point of our analysis is the two-component mixture described in \eqref{eq:start}.
The probability density function (PDF) $f_N$ generates ``normal'' data,  whereas PDF $f_\infty$ generates anomalies, and $0\le p\le 1$ is the mixing parameter.
The ``$\infty$'' symbol is used in $f_\infty$ to convey the intuition that there are arbitrarily many different anomaly types/classes.
In the following, it is assumed that an AFR exists, which is specified in Definition~\ref{def:AFR}.
\begin{definition}[Anomaly-free Region]\label{def:AFR}
	Let $R\subset \mathbb{R}$ and let $f_\infty$ be the unknown anomaly PDF in \eqref{eq:start}. Subset $R$ is an AFR with respect to $f_\infty$ if it holds that
	\begin{equation}
		\int_{R} f_\infty(x)dx = 0.
	\end{equation}
\end{definition}
In other words, an AFR is a region in which there is zero probability that anomalies occur.
Note that this does not impose a restriction on the region where normal data can occur.

The mixture in \eqref{eq:start} can be thought of as combining two distributions via a (biased) coin flip---with probability $p$ the coin lands on ``heads'' and we draw an anomaly from $f_\infty$, and with probability $1-p$ the coin lands on ``tails'' and we draw a normal data point from $f_N$.
Let $B$ be the Bernoulli random variable that describes this coin flip.
By making $B$ explicit, we can rewrite \eqref{eq:start} as 
\begin{align}
	g(x,B) &= \bigl((1-p)f_N(x,\bm{\theta})\bigr)^{1-B}\bigl(p f_\infty(x)\bigr)^{B}.\label{eq:mixture_better}
\end{align}

Next, the mixture in \eqref{eq:mixture_better} needs to be connected with a dataset.
Let $\bm{x}=[x_1,\mathellipsis,x_n]\in\mathbb{R}^n$ be a dataset obtained by drawing $n$ data points independently from $g(\cdot|B)$ according to a sequence of iid Bernoulli variables $\bm{B}=[B_1,\mathellipsis,B_n]$.
Because of independence, the likelihood of the parameters given dataset $\bm{x}$ is given by 
\begin{equation}\label{eq:likelihood_impossible}
	\mathcal{L}(\bm{\theta},p,\bm{B}|\bm{x}) = \prod_{t=1}^{n}  \bigl((1-p)f_N(x_t,\bm{\theta})\bigr)^{1-B_t}\bigl(p f_\infty(x_t)\bigr)^{B_t}.
\end{equation}
For each $x_t$ in $\bm{x}$, the corresponding $B_t$ indicates if $x_t$ was drawn from $f_N$ or $f_\infty$, i.e., if $x_t$ is normal or an anomaly.
Hence, the underlying anomaly detection problem is solved if we figure out $\bm{B}$.
However, it is likely impossible to optimize $\bm{B}$ without knowing anything about $f_\infty$.
To resolve this dilemma, one can replace $\bm{B}$ with an estimate $\hat{\bm{B}}$.
This way, one can rewrite the likelihood in ~\eqref{eq:likelihood_impossible} as
\begin{equation}\label{eq:likelihood}
	\mathcal{L}(\bm{\theta},p|\hat{\bm{B}},\bm{x}) = \prod_{t=1}^{n}  \bigl((1-p)f_N(x_t,\bm{\theta})\bigr)^{1-\B_t}\bigl(p f_\infty(x_t)\bigr)^{\B_t}		
\end{equation}
which is the same as \eqref{eq:likelihood_impossible} except that $\bm{B}$ is replaced with its estimate $\hat{\bm{B}}$, and that this estimate is given. Notice that in likelihood function~\eqref{eq:likelihood} the anomaly distribution $f_\infty$ has no learnable parameters, which has the advantage that the maximum of $\mathcal{L}$ only depends on $p$, $\bm{\theta}$. 

\subsection{Constrained Likelihood}

Directly optimizing the likelihood in~\eqref{eq:likelihood} is not advisable since the result may violate two important constraints:
\begin{enumerate}
	\item \textbf{Valid parameter constraint}: We must ensure that only permitted values of $\bm{\theta}$ and $p$ are considered during the optimization, i.e., $\bm{\theta}$ must be a real-valued vector/matrix and $0\le p\le 1$.
	\item \textbf{Consistency constraint}: 
    We must ensure that $\bm{\theta}$ and $p$ are such that the probability measure of $g$ inside $R$ is consistent with the number of data we observe inside $R$ due to the LLN.
\end{enumerate}	
In contrast to the first constraint, the second is less intuitive.
To further motivate the necessity of the consistency constraint, note that the probability of observing data outside $R$ is given by
\begin{align}\label{eq:anomaly_prob}
	\P(x\not\in R) &= 1-\int_{R} g(x) = (1-p)\left(1-\int_{R}f_N(x,\bm{\theta})\right) + p \nonumber\\
	&= 1-(1-p)\int_{R}f_N(x,\bm{\theta})=P.
\end{align}
If $\bm{x}$ is drawn from $g$, then due to the LLN $P$ can be estimated via
\begin{equation}
	P=P(\bm{\theta},p)\approx \hat{P}=\frac{1}{n}\sum_{t=1}^{n}\mathbf{1}_{\overline{R}}(x_t)\label{eq:P_hat}
\end{equation}
where $\mathbf{1}_{\overline{R}}(\cdot)$ is the indicator function of the complement of $R$.
The consistency constraint has the purpose to ensure that $\bm{\theta}$ and $p$ are chosen such that the approximation in~\eqref{eq:P_hat} holds.

Due to the high relevance of $P$ for AFR-based anomaly detection, it is important to have a reliable estimator for $P$.
Depending on $R$, the true $P$ may be very close to its extreme values $0$ and $1$.
Since the estimator $\hat{P}$ in \eqref{eq:P_hat} is biased at these extremes, it is better to use the Wilson score method


\begin{align}
	P&\in\bar{P} \pm w \nonumber\\
	\bar{P}&=\frac{1}{1+\frac{z^2}{n}}\left(\hat{P}+\frac{z^2}{2n}\right)\label{eq:P_bar}\\
	w&=\frac{z}{1+\frac{z^2}{n}}\sqrt{\frac{\hat{P}(1-\hat{P})}{n}+\frac{z^2}{4n^2}}\label{eq:w}
\end{align}
where $z$ is the $1-\frac{\alpha}{2}$ quantile of the standard normal distribution for significance level $\alpha$.

Adding the two discussed constraints to likelihood \eqref{eq:likelihood} gives the following optimization problem:
\begin{equation}\label{eq:optimization_problem}
	\begin{aligned}
		\textbf{maximize}\quad& \mathcal{L}(\bm{\theta},p|\hat{\bm{B}},\bm{x},R)\\
		\textbf{subject to}\quad& 0\le p \le 1, \bm{\theta}\in\mathbb{R}^q\\
		&\bar{P}-w\le 1-(1-p)\int_{R}f_N(x,\bm{\theta})dx\le \bar{P}+w
	\end{aligned}
\end{equation}
Solving this optimization problem gives the constrained maximum likelihood estimates of parameters $\bm{\theta}$ and anomaly probability $p$.

\subsection[Maximum Likelihood Estimators]{Maximum Likelihood Estimators for $p$ and $\bm{\theta}$}

As we will see, mixture parameter the maximum likelihood estimate (MLE) of $p$ has the same general form regardless of the underlying normal distribution $f_N$, while the MLE of $\bm{\theta}$ is specific to $f_N$.
In the following, we abbreviate $I=\int_{R}f_N(x,\bm{\theta}) dx$ to keep the notation simple.
\subsubsection{Karush-Kuhn-Tucker conditions}
We can rewrite the constrained optimization problem \eqref{eq:optimization_problem} as a system of (in)equations using the Karush-Kuhn-Tucker conditions:
\begin{subequations}\label{eq:KKT}
	\begin{align}
		L &= L(\bm{\theta},p|\hat{\bm{B}},\bm{x})\nonumber\\
		&=
		\begin{aligned}
			\mathcal{L}(\bm{\theta},p|\hat{\bm{B}},\bm{x}) - \lambda_1 \left(p^2-p\right) -\\ \lambda_2\left(\left(1-(1-p)I- \bar{P}\right)^2-w^2\right)\label{eq:L}
		\end{aligned}\\	
		0&= \frac{\partial \mathcal{L}}{\partial p} - \lambda_1(2p-1) - 2\lambda_2 I(1-(1-p)I-\bar{P})\label{eq:dp}\\
		0&= \frac{\partial \mathcal{L}}{\partial \theta^{(1)}} + 2\lambda_2(1-p)\left(1-(1-p)I-\bar{P}\right)\frac{\partial}{\partial \theta^{(1)}} I\label{eq:dt1}\\
		0&= \frac{\partial \mathcal{L}}{\partial \theta^{(2)}} + 2\lambda_2(1-p)\left(1-(1-p)I-\bar{P}\right)\frac{\partial}{\partial \theta^{(2)}} I\label{eq:dt2}\\	
		\vdots&\quad\quad\vdots\quad\quad\quad\quad\vdots\nonumber\\
		0&= \frac{\partial \mathcal{L}}{\partial \theta^{(q)}} + 2\lambda_2(1-p)\left(1-(1-p)I-\bar{P}\right)\frac{\partial}{\partial \theta^{(q)}} I\label{eq:dtq}\\
		0 &= \lambda_1\left(p^2-p\right) \quad\quad\quad\quad\quad\quad\quad\quad\text{\scriptsize(compl. slack.)}\label{eq:comp_slack_1}\\
		0 &= \lambda_2\left(\left(1-(1-p)I-\bar{P}\right)^2-w^2\right)  \;\;\:\text{\scriptsize(compl. slack.)}\label{eq:comp_slack_2}\\
		0 &\ge p^2 - p\\
		w^2 &\ge \left(1-(1-p)I-\bar{P}\right)^2\\
		0 &\le \lambda_1,\lambda_2
	\end{align}
\end{subequations}
Because of complementary slackness~\cite{robinson1972quadratically}, we must have one of the following four cases:
\begin{enumerate}
	\item $\lambda_1=0$ and $\lambda_2=0$: The standard unconstrained MLEs for $\bm{\theta}$ and $p$ satisfy both constraints.
	\item $\lambda_1\neq0$ and $\lambda_2=0$: The valid parameter constraint requires $p$ to be either $p=0$ or $p=1$, whereas $\bm{\theta}$ can be estimated with the standard MLE since $\bm{\theta}$ does not appear in \eqref{eq:comp_slack_1}.
	\item $\lambda_1=0$ and $\lambda_2\neq0$: The consistency constraint requires $p,\bm{\theta}$ such that  $\left(1-(1-p)I-\bar{P}\right)^2=w^2$
	\item $\lambda_1\neq0$ and $\lambda_2\neq0$: Both constraints are active. The implications are the same as for the previous case with the additional (simplifying) condition that either $p=0$ or $p=1$.
\end{enumerate}
The third case is arguably the most interesting since it frequently occurs for bad estimates $\hat{\bm{B}}$, and since it is hard to solve.

\subsubsection[MLE of p]{MLE of $p$}
Parameter $p$ is the mixture weight of PDF $g$ and thus is equal to the probability of observing an anomaly.
The MLE of $p$ has four possible solution candidates that depend on the cases described above.
In case 1, we have the well-known result $p=\frac{1}{n}\sum_{t=1}^{n}\B_t$, which is the standard unconstrained MLE;
In cases 2 and 4, $p$ must either be $p=0$ or $p=1$ to satisfy the first constraint.
Case 3 has the following solution:
\begin{theorem}\label{thm:p}
	For $\lambda_1=0$ and $\lambda_2\neq 0$, the equation system~\eqref{eq:KKT} has the following solution for $p$
	\begin{equation}\label{eq:p_mle_constrained}
		p=\frac{1}{n-\Omega}\sum_{t=1}^{n}\B_t
	\end{equation}
	where $\Omega<n-\sum_{t=1}^{n}\B_t$ is the density surplus gradient, which is given by
	\begin{equation}\label{eq:omega}
		\Omega=
  \frac{\frac{\partial}{\partial \theta^{(1)}}\log\mathcal{L}}{\frac{\partial}{\partial \theta^{(1)}}\log I}I = \mathellipsis = \frac{\frac{\partial}{\partial \theta^{(q)}}\log\mathcal{L}}{\frac{\partial}{\partial \theta^{(q)}}\log I}I
	\end{equation}
 and where (either of the) derivatives of the right-hand side are evaluated at an MLE of $\theta$.
\end{theorem}

Thus, the constrained solution for $p$ is similar to the unconstrained MLE, with the only difference being the subtraction of $\Omega$ in the denominator. Density surplus gradient $\Omega$ is 0 in the unconstrained case; positive if the $p_\text{MLE}$ is too ``small''; and negative if $p_\text{MLE}$ is too ``large'' to fulfil the consistency constraint.

\subsubsection[MLE of theta]{MLE of $\bm{\theta}$ under the Gaussian assumption}\label{sec:theta}
Since $f_N$ can be arbitrary, $\bm{\theta}$ has a different maximum likelihood solution for most distributions.
In the following, we discuss the MLE for Gaussian $f_N$ with $\bm{\theta}=[\mu,\sigma^2]$ under the assumption that AFR $R$ is an interval $[a,b]$.
Other distributions are discussed in the Appendix.

\begin{theorem}\label{thm:sigma2}
	For $\lambda_2\neq0$, Gaussian $f_N=\mathcal{N}(\mu,\sigma^2)$ and $R=[a,b]$, equation system~\eqref{eq:KKT} has the following solution for $\sigma^2$
	
	\begin{equation}\label{eq:sigma2_solution}
		\sigma^2=\frac{\displaystyle \beta m}{\displaystyle m+\beta W_r\left(\gamma\right)}
	\end{equation}
	with 
	\begin{subequations}
		\begin{align}
			\bar{x}&=\frac{1}{n-\sum_{t=1}^{n}\B_t}\sum_{t=1}^{n}(1-\B_t)x_t\label{eq:x_bar}\\
			\bar{x^2}&=\frac{1}{n-\sum_{t=1}^{n}\B_t}\sum_{t=1}^{n}(1-\B_t)x_t^2\label{eq:x2_bar}\\
			m&=\frac{1}{2}\left((a-\mu)^2-(b-\mu)^2\right)\label{eq:m}\\
			\alpha &= \bar{x^2} - \mu\bar{x}+(\mu-\bar{x})a\label{eq:alpha}\\
			\beta &= \bar{x^2} - \mu\bar{x}+(\mu-\bar{x})b\label{eq:beta}\\
            \gamma&=\left(\frac{\alpha-\beta}{\beta^2}\right)me^{-\frac{m}{\beta}}
		\end{align}
	\end{subequations}
	and where $W_r(\cdot)$ is the $r$-Lambert function~\cite{mezHo2017generalization} with $r=-\frac{\alpha}{\beta}e^{-\frac{m}{\beta}}$ .
\end{theorem}

Theorem~\ref{thm:sigma2} gives an analytic expression for variance $\sigma^2$ if the consistency constraint is active---if it is not, $\sigma^2$ is given by the standard Gaussian variance estimator. Since the $r$-Lambert function has up to three real branches, \eqref{eq:sigma2_solution} gives up to three solution candidates for $\sigma^2$. To determine the correct solution, we can leverage \eqref{eq:omega}: Only one of the three branches will yield a real-valued $\sigma$ such that $\Omega_\mu = \Omega_{\sigma^2}$. 

\begin{figure}[!t]
\centering 

\includegraphics[scale=0.6]{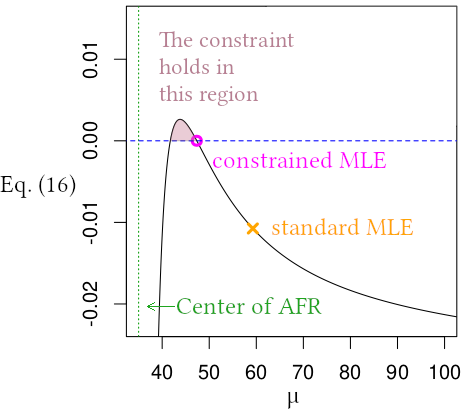}

\caption{Illustration of the estimation procedure for $\mu$. Since Eq.~\eqref{eq:mu} is quasi-concave, it is easy to find the constrained MLE, which is located at the function's root that is closer to the unconstrained MLE. The center of the AFR is at $\frac{a+b}{2}$, and there we have $\frac{\partial L}{\partial \mu}=0$, so Eq.~\eqref{eq:mu} is undefined at $\frac{a+b}{2}$.}
\label{fig:constraint}

\end{figure}
There is no analytic expression for mean $\mu$ because this would require an analytic expression for $\int_{a}^{b}e^{-x^2} dx$, which does not exist.
However, a careful study of equation system~\eqref{eq:KKT} reveals a straightforward scheme for accurately approximating $\mu$:
Let $s_B=\sum_{t=1}^n \B_t$, let $\Phi(\cdot)$ be the standard Gaussian cumulative distribution function, and let $\Phi_\theta(x) = \Phi\left(\frac{(x-\mu)(m+\beta W_r(\gamma))}{\beta m}\right)$.
Theorems \ref{thm:p} and \ref{thm:sigma2} allow us to write the constraint equation~\eqref{eq:comp_slack_2} as a function of $\mu$.

\begin{equation}\label{eq:mu}
 h(\mu|\bm{x},\hat{\bm{B}},R)=w^2 \shortminus \left(1\shortminus\left(1\shortminus\frac{s_B}{n\shortminus\Omega_\mu}\right) (\Phi_\theta(b)\shortminus\Phi_\theta(a))\shortminus\bar{P}\right)^2
\end{equation}
This function is real-valued for $\mu\in (-\infty;\frac{a+b}{2})$ if $\bar{x} < \frac{a+b}{2}$, or else for $\mu\in (\frac{a+b}{2};\infty)$. Importantly, $\eqref{eq:mu}$ is quasi-concave in its real domain. Hence, it is easy to find its root(s) using standard optimization procedures, e.g., numerical gradient or binary search. In the case where \eqref{eq:mu} has two roots, one can determine the better solution using the log-likelihood function $\ell$. An illustration is depicted in Figure~\ref{fig:constraint}.

	\section{Implementation}
	
		

	
	

	After presenting the theoretical foundation of AFRs and constrained AD, we now focus on practical aspects.
    There are many different ways to implement AFR-based AD.
    Hypothetically speaking, most existing detectors could be equipped with an AFR by leveraging the detector's output to construct estimate $\hat{\bm{B}}$, e.g., via Otsu's thresholding method~\cite{otsu1979threshold}, and plug that into Equation \eqref{eq:mu}.
    However, for complicated detectors, e.g., based on deep learning, it is unclear how the AFR affects the detector, so we instead propose a more direct scheme.

    \subsection{Algorithm}\label{sec:alg}
    In general, constrained AD requires four inputs: i) dataset $\bm{x}$; ii) AFR $R$; iii) anomaly label estimate $\hat{\bm{B}}$, iv) significance level $\alpha$.\footnote{Note that  $\alpha$ is non-critical, since it only affects the speed of the constrained estimation, but not the accuracy.}
    Estimate $\hat{\bm{B}}$ can be arbitrary, i.e., it does not have to be a good estimate.
    We propose to select several $\hat{\bm{B}}$ randomly, e.g. $n_B=5$, and compute the average deviation from the mode of $f_N$ over all $\hat{\bm{B}}$.
    Thus, we can provide $n_B$ instead of a concrete $\hat{\bm{B}}$ as input.
    The necessary steps to achieve this are listed in Algorithm~\ref{alg:ours}.
    
	
	\begin{algorithm}
		\caption{Constrained Anomaly MLE (CAMLE)}
		\label{alg:ours}
		\begin{algorithmic}[1]
			\State\textbf{Input} $\bm{x}$, $R$, $n_B$, $\alpha$
            \State Use $\alpha$ to compute $\bar{P},w$ as in Equations \eqref{eq:P_hat}, \eqref{eq:P_bar}, \eqref{eq:w}
            \State Score$_t\leftarrow 0 \quad\forall t\in 1,\mathellipsis,n$ 
            \For{$j \in 1,\mathellipsis,n_B$}
            \State Randomly initialize $\hat{\bm{B}}$
            \State $\B_t \leftarrow 0 \quad\forall x_t \in R$
            \State $p,\bm{\theta}\leftarrow$ standard MLE$(\bm{x},\hat{\bm{B}})$
            \If{$\left(1-(1-p)\int_{R}f_N(x,\bm{\theta})dx-\bar{P}\right)^2 > w^2$}
			\State $p,\bm{\theta}\leftarrow$ constrained MLE$(\bm{x},\hat{\bm{B}},R,\bar{P},w)$
            \EndIf
            \State Score$_t\leftarrow$ Score$_t+\max_\nu f_N(\nu,\bm{\theta}) - f_N(x,\bm{\theta})\quad \forall t$
             \EndFor
            \State Score$_t\leftarrow \frac{1}{n_B}$ Score$_t\quad \forall t$
			 \State\textbf{Output} Score
		\end{algorithmic}
		
	\end{algorithm}

    Conceptually, Algorithm~\ref{alg:ours} (CAMLE) is comparable to bootstrap aggregation. It draws $n_B$ subsets of data outside the AFR, joins them with data inside the AFR, estimates the parameters, and finally computes an average over all subsets.
    CAMLE has an asymptotic runtime of $\mathcal{O}(n_B\cdot n)$. In most realistic settings, $n_B$ can be set to a small constant, which leads to linear runtime in total.

    \subsection{Estimation of AFR}
    In some settings, a valid AFR is not available and needs to be estimated. Depending on the setting, there are different ways to achieve this. We discuss the following three scenarios.
    
    \subsubsection{Labels available for normal and anomaly class}
    In some settings a few labeled data points of both classes are available.
    Here, a valid AFR can be approximated as any region that contains no labeled anomalies and no unlabeled data points.
    We recommend selecting a region that contains several labeled normal data points, since this will typically improve the quality of the constrained AD solution due to the estimation of anomaly proportion $p$ being more accurate.

    \subsubsection{Labels available only for normal class}
    If labeled normal and unlabeled data points are available, but no labeled anomalies, then the AFR can be estimated via a region that exclusively contains labeled normal data. Here it is reasonable to prefer regions that are remote from unlabeled data, if possible.

    \subsubsection{No labels available}
    In unsupervised settings, no labels are available. Here, there are two options: a) Select an empty region as AFR; b) Include unlabeled data in the AFR. Option a) is safer since it avoids the possibility of falsely including an anomaly in the AFR. However, option b) has higher statistical efficiency and should be preferred if one can make the mild assumption that central, high-density regions of the data space are unlikely to contain anomalies.
    This is also the approach that we use in the benchmark tests in the following section.
	\section{Experiments}
	The theoretical findings and the algorithm presented in this article motivate us to study pertinent and testable research questions.
	
	\begin{enumerate}
		\item [\textbf{RQ1}.] \textbf{Simulation}: How does the estimation performance of the constrained MLE compare against the unconstrained MLE in a controlled setting?
		\item [\textbf{RQ2.1}.] \textbf{Unsupervised benchmark test}: How strong is CAMLE's detection performance on popular benchmark datasets where no AFR is known and has to be estimated?
        \item [\textbf{RQ2.2}.] \textbf{Unsupervised sensitivity analysis}: How does the AFR estimation affect CAMLE's detection performance?
        \item [\textbf{RQ3}.] \textbf{Semi-supervised benchmark test}: How does CAMLE's detection performance compare against other semi-supervised methods on a benchmark dataset with a valid AFR?
		
	\end{enumerate}
	For each research question, a separate experiment is conducted.
	In the following, motivation, setup, data, baselines and results of each experiment are described separately. A combined discussion of all research questions comes afterward.
    All data, descriptions, runtimes and source code can be privately accessed online.\footnote{This article's supplementary material is available at \url{https://github.com/mtoller/afr}}
	
	\subsection{RQ1: Simulation}

    \subsubsection{Motivation}
    The constrained estimators presented above are optimal in the AFR-constrained AD setting, but the merit of the constraint itself needs to be assessed separately.
    The AFR encodes information about anomaly class $f_\infty$, yet it is not obvious that this information is useful for estimating the parameters of normal class $f_N$.
    Moreover, from a theoretical viewpoint, it is interesting to study scenarios in which the classical MLE is sub-optimal.

    \subsubsection{Setup}
    We start with a controlled experiment on data drawn from a known distribution.
    Hence, we use the following simulation procedure:
    We select a Gaussian distribution for $f_N$ and a 50:50 mixture of two uniform distributions for $f_\infty$.
    We uniformly draw 100 different ground-truth parameter combinations from the following ranges: $\mu\in[-5;5],\sigma\in[0.1;2],p\in[0.05;0.95]$.
    The AFR is then fixed to $[\mu-0.98\sigma;\mu+0.99\sigma]$ and the limits of the $f_\infty$ mixture are set to $[\mu-10\sigma;a]$ and $[b;\mu+10\sigma]$, respectively.
    Then, for each ground truth combination $(\mu,\sigma,p,a,b)$, we independently draw 100 datasets consisting of 1000 samples, and 10 different ``guessed'' label sets $\hat{\bm{B}}\sim \text{Bernoulli}(p)$ for each dataset.
    On each resulting dataset, the unconstrained MLE and the constrained MLE of each parameter are first computed using the ground truth labels $\bm{B}$, and then using the 10 ``gussed'' label sets $\hat{\bm{B}}$, respectively.
    The estimates are compared against the ground truth via the median absolute deviation (MAD), and the average over all datasets is computed. We prefer the MAD over the classical mean-squared error since in rare cases the estimation procedure for the r-Lambert function converges to the wrong branch due to floating point inaccuracy.

    \subsubsection{Results}
    The results of the simulation are listed in Table~\ref{tab:simulation}. When the true labels were provided, both constrained and unconstrained MLE performed similarly and achieved small errors. When guessed labels were provided, CAMLE achieved smaller errors than unconstrained MLE for $\mu$ and $\sigma$, but yielded larger errors for the mixture fraction $p$.
    
\smallskip

    In conclusion, this experiment suggests that constrained MLE improves upon unconstrained MLE in a suitable controlled setting where all assumptions are fulfilled.

    \begin{table}[!htb]
    \caption{Median Absolute Deviation between ground truth parameters and estimates. A negative error difference indicates that the unconstrained MLE has lower error, whereas a positive error difference indicates that the constrained MLE is better.}
    \label{tab:simulation}
    \begin{tabular}{l|ccc|ccc}
    &\multicolumn{3}{c}{True labels} & \multicolumn{3}{c}{Guessed labels}\\
    \toprule
    & $\mu$ & $\sigma$ & $p$ & $\mu$ & $\sigma$ & $p$\\
    \midrule
    \textbf{MLE} &0.0409&0.0284 &0.0134 &0.0510 & 1.6435&0.0426\\
    \textbf{CAMLE} & 0.0410& 0.0284& 0.0134&0.0176  &0.6897 & 0.1314 \\
    \bottomrule
    \textbf{Err. Diff.} & -0.001 & 0.000 & 0.000& 0.0334 & 0.9538 & -0.0888
    \end{tabular}
    \end{table}

    \subsection{RQ2.1: Unsupervised Benchmark Test}
	\subsubsection{Motivation}
    Having seen that constrained MLE is promising in a controlled setting, it is natural to study its empirical performance and to see how it compares against widely-used AD methods. 
	In data mining research, it is conventional to conduct such a study on a large variety of diverse AD benchmark datasets. 
    Due to this diversity, we cannot expect that CAMLE's assumptions will be satisfied during the experiment, so the theoretical guarantees will not hold.
    However, it is interesting to observe how CAMLE behaves in practice--especially when no AFR is known and has to be estimated.

    \begin{table*}
    \caption{Results of the unsupervised benchmark test. Depicted values in rows starting with dataset names are AUC-ROC scores.}
    \label{tab:benchmar_results}
    \centering
    \begin{tabular}{l|ccc|ccc|cc}
    Method class&\multicolumn{3}{c|}{Heuristical} & \multicolumn{3}{c|}{Probabilistic} & \multicolumn{2}{c}{Deep learning}\\
    \textbf{Data / Measure}& \textbf{LOF} & \textbf{IForest }& \textbf{COPOD} & \textbf{EM} & \textbf{MLE }& \textbf{CAMLE}& \textbf{AutoEnc }& \textbf{DeepSVDD}\\
    \toprule
  Annthyroid & 0.68 & 0.82 & 0.78 & 0.68 & 0.68 & 0.96 & 0.67 & 0.72 \\ 
  Cardio & 0.51 & 0.93 & 0.92 & 0.91 & 0.92 & 0.71 & 0.95 & 0.31 \\ 
  Cardiotocography & 0.55 & 0.69 & 0.66 & 0.64 & 0.70 & 0.68 & 0.75 & 0.53 \\ 
  Letter & 0.89 & 0.63 & 0.56 & 0.52 & 0.53 & 0.56 & 0.52 & 0.53 \\ 
  Satimage & 0.62 & 0.99 & 0.97 & 0.95 & 0.95 & 0.95 & 0.98 & 0.39 \\ 
  Vowels & 0.87 & 0.77 & 0.50 & 0.55 & 0.56 & 0.59 & 0.60 & 0.40 \\ 
  Waveform & 0.61 & 0.73 & 0.73 & 0.63 & 0.61 & 0.52 & 0.64 & 0.55 \\ 
  Wilt & 0.61 & 0.45 & 0.34 & 0.33 & 0.35 & 0.39 & 0.34 & 0.47 \\ 
  Yeast & 0.47 & 0.40 & 0.38 & 0.40 & 0.42 & 0.44 & 0.40 & 0.49 \\ 
  \midrule
  ø Rank (per method class) & 2.11 & 1.56 & 2.33 & 2.56 & 1.89 & 1.56 & 1.44 & 1.56\\
  \midrule
  ø Rank (overall) &4.11  &    2.56   &   \multicolumn{1}{c}{4.56}  &    6.11   &   4.67   &   \multicolumn{1}{c}{4.11}   &   4.22  &    5.67 \\
  Total runtime (seconds)  &0.53   &  8.81   &   \multicolumn{1}{c}{0.74}   &   0.17   &   0.70   &   \multicolumn{1}{c}{6.05}  &   435.27  &  62.28 
    \end{tabular}
    
    \end{table*}

     \subsubsection{Datasets}
     We use two different groups of datasets, namely \textit{development datasets} and \textit{evaluation datasets}.
     The development datasets are used for preliminary analyses of the approaches, e.g., does the detector encounter an error, how many initial $\hat{\bm{B}}$ should be sampled, and so on. Evaluation datasets are used for the quantitative comparison.
     The rationale for using two different groups is to avoid data dredging (see \cite{andrade2021harking}).
     We aim to ensure that the performance of our method is not specific to the datasets on which we conducted method development.
     
     We use the following publicly available datasets as \textit{development datasets}:
	\textsc{Glass}~\cite{keller2012hics};
	\textsc{Heart}~\cite{dua2019uci};
	\textsc{Hep}~\cite{dua2019uci};
	\textsc{Ion}~\cite{keller2012hics};
	\textsc{Lymph}~\cite{lazarevic2005feature};
	\textsc{Park}~\cite{dua2019uci};
	\textsc{Pima}~\cite{dua2019uci};
	\textsc{Stamps}~\cite{micenkova2015stamp};
	\textsc{Shuttle}~\cite{zhang2009new};
	\textsc{WPBC}~\cite{dua2019uci}.
    We select the following datasets as \textit{evaluation datasets}: \textsc{Annthyroid}~\cite{quinlan1986induction}; \textsc{Cardio}~\cite{ayres2000sisporto}; \textsc{Cardiotocography}~\cite{ayres2000sisporto}; \textsc{Letter}~\cite{frey1991letter}; \textsc{Satimage}; \textsc{Vowels}~\cite{kudo1999multidimensional}; \textsc{Waveform}~\cite{loh2011classification}; \textsc{Wilt}~\cite{campos2016evaluation}; \textsc{Yeast}~\cite{horton1996probabilistic}. 
    Further information about these datasets such as number of samples and description of normal and anomaly classes, as well as several other details can be found in \cite{han2022adbench}.

	\subsubsection{Compared Detectors}
	The following heuristic detectors from the AD literature are compared (presented in chronological order):
	\begin{enumerate}
		\item \textit{LOF}: Local outlier factor---a local neighborhood and density-based anomaly detector~\cite{breunig2000lof}. Five nearest neighbors and Minkowski distance are used.
		\item \textit{IForest}: Isolation Forest, which uses random hyper-rectangles to detect isolated data~\cite{liu2008isolation}. 1000 random trees are used in each computation.
		\item \textit{COPOD}: Copula-based outlier detection---a heuristic method to determine data extremeness via tail probabilities~\cite{li2020copod}. This method has no hyperparameters.\footnote{Despite its sound theoretical basis, we classify COPOD as heuristic since the algorithm described in \cite{li2020copod} does not use a copula.}
    \end{enumerate}
    The following probabilistic approaches are included in the comparison:
        \begin{enumerate}
        \item \textit{EM}: Expectation maximization with a Gaussian and a uniform component, and with discrete labels. This particular setup of expectation maximization is conceptually similar to an unconstrained version of our method where the labels are updated in a greedy classical MLE sense. 
        \item \textit{MLE}: The same as Alg.~\ref{alg:ours}, except that it always returns the unconstrained solution. Scores are computed for each dimension and then aggregated via the average over all dimensions. $n_B=5$ random initializations are used. The AFR $[a;b]$ is inferred from the 0.24 and 0.75 quantiles, respectively. Significance level $\alpha$ is set to $0.05$.
		\item \textit{CAMLE}: Our method as presented in Alg.~\ref{alg:ours}. Scores are computed for each dimension and then aggregated via the average over all dimensions.
		$n_B=5$ random initializations are used. The AFR $[a;b]$ is inferred from the 0.24 and 0.75 quantiles, respectively.
		Significance level $\alpha$ is set to $0.05$.
	\end{enumerate}
    For MLE and CAMLE, we choose the lower boundary of the AFR to be the 0.24 quantile instead of the more common 0.25 quantile to decrease the chance that $\bar{x}$ coincides with $\frac{a+b}{2}$, which has no constrained solution.
    
    The following deep-learning AD methods are compared:
    \begin{enumerate}
        \item \textit{AutoEnc}: An autoencoder network~\cite{le1987modeles} whose reconstruction error is used as anomaly score. The hidden layers have 1\textrightarrow64\textrightarrow32\textrightarrow32\textrightarrow64\textrightarrow1 neurons and use $\tanh$ as activation function, and the network was trained using Adam with a dropout rate of 20\%.
        \item \textit{Deep-SVDD}: SVDD~\cite{tax2004support} applied to the latent representation of an autoencoder~\cite{ruff2018deep}. The hidden  and output layers use ReLU and sigmoid activation functions, respectively, and the network was trained using Adam with a dropout rate of 20\%.
        
    \end{enumerate}
    
	
	

	\subsubsection{Evaluation Setup}
	The predictions of every detector are collected from every evaluation dataset.
	These predictions are then compared with the true anomaly labels using the area under curve of the receiver-operator characteristic (AUC-ROC), which is an estimate of the probability that a detector will rank a randomly chosen anomaly higher than a randomly chosen non-anomaly.
    It is computed via $\text{AUC-ROC}=\int_{-\infty}^{\infty}\text{TPR}(\tau)\text{FPR}'(\tau) d\tau$ where TPR is the true positive rate $\frac{\text{TP}}{\text{TP}+\text{FP}}$, FPR the false positive rate $\frac{\text{FP}}{\text{FP}+\text{TN}}$, and $\tau$ is a threshold that separates the scores of anomalies and non-anomalies.
    We also track each method's total execution runtime to assess the practical efficiency of all implementations.

	\subsubsection{Results}
	The results of the conducted experiment are shown in Table~\ref{tab:benchmar_results}.
    The best average rank among heuristic methods was achieved by IForest; among probabilistic methods by CAMLE; among deep-learning methods by AutoEnc.
    The best average rank in total was achieved by IForest, with CAMLE and LOF sharing the second place.
    The unconstrained MLE achieved a worse average rank than the constrained MLE. 
	The fastest method was EM, followed by LOF.
    The deep-learning based methods AutoEnc and DeepSVDD had the longest and second longest execution runtime, respectively.

    \smallskip
 
	Effectively, these results highlight that an AFR-based AD method is capable of outperforming several widely-used AD methods on selected benchmark datasets even when no AFR is known and has to be estimated.

	\subsection{RQ2.2: Unsupervised Sensitivity Analysis}
    \subsubsection{Motivation}
    It remains unclear how sensitive CAMLE is to changes in the AFR estimation procedure.
    In particular, it is interesting whether CAMLE can retain its high performance across many different AFRs, or if only a small number of specific AFRs produce high results.
    
    \subsubsection{Setup}
    For this analysis, we use the \textsc{Annthyroid} dataset. For other datasets, please refer to the Appendix.
    We initialize 1100 different AFRs (per dimension of the dataset) using the following procedure:
    \begin{enumerate}
    \item Eleven equidistant offsets $\Delta$ are selected from the range [0.01;0.99], i.e, 0.010, 0.108,...,0.892, 0.990.
    \item For each $\Delta$, 100 equidistant $a,b$ pairs are selected as empirical quantiles, ranging from 0.0001 to 0.999.
    \end{enumerate}
    We aggregate CAMLE's predictions scores on \textsc{Annthyroid} for each $\Delta$ using the average over all $a,b$ pairs.

    \subsubsection{Results}
    The results of this analysis are depicted in Figure~\ref{fig:ablation}. CAMLE retained a score higher than the second highest-scoring method (IForest) for $0.108 \le \Delta < 0.892$. Outside this region, CAMLE's performance dropped below an AUC-ROC of 0.82.

    \smallskip
    
    These results suggest that CAMLE is not sensitive to the AFR estimation procedure and that its performance does not drop significantly.

    \begin{figure}
        \centering
        \includegraphics[scale=0.6]{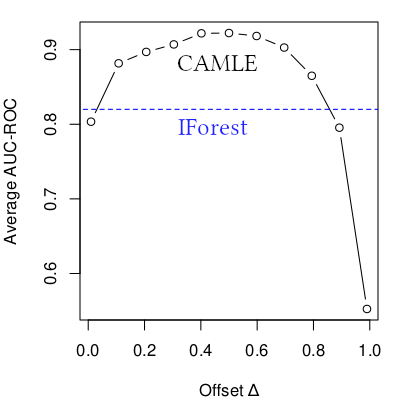}
        \caption{Results of the sensitivity analysis. CAMLE was computed with 1100 different AFRs per dimension of dataset \textsc{Annthyroid}. CAMLE retaints the highest scores among all compared detectors for AFR lengths $0.108\le\Delta<0.892$.}
        \label{fig:ablation}
    \end{figure}

    \subsection{RQ3: Semi-Supervised Benchmark Test}
    \subsubsection{Motivation}
    If a valid AFR is given, this is a different setup and needs to be investigated separately.
    The simulation study alone is not sufficient since it does not cover CAMLE's empirical performance, and RQ2 did not cover the case that an AFR is given.

    \subsubsection{Dataset}
    To the best of our knowledge, there is no publicly available benchmark dataset with a known AFR.
    Hence, we collect such a dataset, and refer to it as \textsc{Office}.
	We collected \textsc{Office} from an office worker's time sheets and interviewed the worker to determine anomaly labels.
	The worker's personal details as well as all company information were removed to preserve anonymity.
	\textsc{Office} consists of 365 data points---1 per day for a full year, including weekends with 0 expected working time. 
    Each data point is the number of minutes by which the actual working time deviates from this day's expected working time.
	AFR $R$ is $\pm$ 29 minutes since deviations shorter than 30 minutes are considered normal by the management regardless of what caused them.
	There are 298 normal data inside $R$, 54 normal data outside $R$, and 13 labeled anomalies.
	We publish this dataset along with this article.

    \subsubsection{Setup}
    We compare CAMLE with semi-supervised methods PReNet~\cite{pang2023deep} and Overlap~\cite{jiang2023anomaly}.
    For both semi-supervised methods, we re-use the original implementations.
    To analyse different setups, PReNet and Overlap receive 10\%, 20\% and 50\% of \textsc{Office} as labeled training data, respectively.
    CAMLE receives AFR=$[a;b]=[-29;29]$ as input.
    We compute the AUC-ROC score for each setting.

    \subsubsection{Results}
    The results of this experiment are depicted in Table~\ref{tab:semi}.
    CAMLE obtains higher scores than PReNet and Overlap at all training sizes. The scores of PReNEt and Overlap increase as the training size increases.

    \smallskip

    These results suggest that an AFR-based method can outperform semi-supervised methods if a valid AFR is provided.

    \begin{table}[!ht]
        \centering
        \caption{AUC-ROC scores on \textsc{Office}. As the training size increases, semi-supervised  methods PReNet and Overlap achieve better scores. CAMLE achieves the highest scores with just AFR $[a;b]$ provided.}
        \label{tab:semi}
        \begin{tabular}{l|ccc|ccc|c}
             \textbf{Method} & \multicolumn{3}{c|}{\textbf{PReNet}}&\multicolumn{3}{c|}{\textbf{Overlap}}&\textbf{CAMLE}\\
             \midrule
             Train. size & 10\% & 20\% & 50\% & 10\% & 20\% & 50\% & AFR\\
             \midrule
             AUC-ROC & 0.66 & 0.92 & 0.96 & 0.66 &0.66& 0.72 & \textbf{0.97}
        \end{tabular}
        
    \end{table}

    \section{Discussion}
    \subsection{Implications}
    \subsubsection{RQ1}
    The simulation study confirms that the constrained MLE is superior over the unconstrained MLE at estimating $f_N$'s parameters $\mu,\sigma$ when a valid AFR is known and $\hat{\bm{B}}\neq \bm{B}$.
    It is evident that the constrained estimation of $p$ in this setting must be inferior to the unconstrained estimation to compensate for the other improved estimates.
    Otherwise we could obtain a solution with higher likelihood, given the AFR, which is not possible since the classical MLE already maximizes the likelihood.

\subsubsection{RQ2}
    The unsupervised benchmark test suggests that CAMLE is superior over its unconstrained variant MLE on the majority of the evaluation datasets.
    Moreover, CAMLE is capable of outperforming several baseline methods from the literature, even if CAMLE's underlying assumptions are not fulfilled by the benchmark datasets (e.g. no AFR known, non-Gaussianity, high-dimensional datasets).
    The sensitivity analysis confirms that a reasonable AFR can be estimated if none is given, which is useful for practical data mining.
    

    \subsubsection{RQ3:} The semi-supervised benchmark test suggests that AFRs are more efficient than training labels at encoding information about the anomaly class.
    The semi-supervised competitors PReNet and Overlap seem to require a large number of training labels to become competitive with CAMLE.
    It seems likely that these methods would eventually outperform CAMLE if the training size is increased beyond 50\%.
    However, in a real-world AD task, it is not realistic that so many labeled training instances are available. 
    \subsection{Limitations}
    The results of our experiments are limited to the small selection of included datasets, as well as to the considered parameter ranges.
    It was assumed that all data within a dataset are independent and identically distributed.
    Dependent data types such as \textit{time series}, \textit{graphs}, \textit{spatial}, \textit{text}, and \textit{images} are not covered.
    The aggregation-based extension of Alg.~\ref{alg:ours} to multi-dimensional datasets used for RQ2 assumes that there is no dependency between dimensions.
    This assumption is violated on many datasets.

    \section{Conclusion and Outlook}
	We introduced the novel concept of anomaly-free regions (AFR) to
improve anomaly detection.
 AFRs and the associated constraints are a fundamental addition to classical anomaly detection.
    AFRs allow one to encode available information about the anomaly class, and can be soundly estimated if no such information is available.
    The resulting constrained form of anomaly detection improves upon the classical unconstrained anomaly detection, and hence is a promising direction for a new class of detectors, e.g., constrained one-class classifiers.

   Despite this paper's focus on anomaly detection, we believe the theoretical contributions will have an impact on other research fields in the future.
    For example, our theoretical contributions can be applied in classification, as the AFR-concept is equivalent to a region that does not contain one (out of several) classes.
    The same generalization can be made for other tasks such as clustering and novelty detection.


	

	\bibliographystyle{IEEEtran}
	\bibliography{references_kdd24}

\begin{thebibliography}{10}
\providecommand{\url}[1]{#1}
\csname url@samestyle\endcsname
\providecommand{\newblock}{\relax}
\providecommand{\bibinfo}[2]{#2}
\providecommand{\BIBentrySTDinterwordspacing}{\spaceskip=0pt\relax}
\providecommand{\BIBentryALTinterwordstretchfactor}{4}
\providecommand{\BIBentryALTinterwordspacing}{\spaceskip=\fontdimen2\font plus
\BIBentryALTinterwordstretchfactor\fontdimen3\font minus \fontdimen4\font\relax}
\providecommand{\BIBforeignlanguage}[2]{{%
\expandafter\ifx\csname l@#1\endcsname\relax
\typeout{** WARNING: IEEEtran.bst: No hyphenation pattern has been}%
\typeout{** loaded for the language `#1'. Using the pattern for}%
\typeout{** the default language instead.}%
\else
\language=\csname l@#1\endcsname
\fi
#2}}
\providecommand{\BIBdecl}{\relax}
\BIBdecl

\bibitem{ruff2021unifying}
L.~Ruff, J.~R. Kauffmann, R.~A. Vandermeulen, G.~Montavon, W.~Samek, M.~Kloft, T.~G. Dietterich, and K.-R. M{\"u}ller, ``A unifying review of deep and shallow anomaly detection,'' \emph{Proceedings of the IEEE}, vol. 109, no.~5, pp. 756--795, 2021.

\bibitem{breunig2000lof}
M.~M. Breunig, H.-P. Kriegel, R.~T. Ng, and J.~Sander, ``Lof: identifying density-based local outliers,'' in \emph{Proceedings of the 2000 ACM SIGMOD international conference on Management of data}, 2000, pp. 93--104.

\bibitem{tax2004support}
D.~M. Tax and R.~P. Duin, ``Support vector data description,'' \emph{Machine learning}, vol.~54, no.~1, pp. 45--66, 2004.

\bibitem{liu2008isolation}
F.~T. Liu, K.~M. Ting, and Z.-H. Zhou, ``Isolation forest,'' in \emph{2008 eighth ieee international conference on data mining}.\hskip 1em plus 0.5em minus 0.4em\relax IEEE, 2008, pp. 413--422.

\bibitem{toller2021cluster}
M.~B. Toller, B.~C. Geiger, and R.~Kern, ``Cluster purging: Efficient outlier detection based on rate-distortion theory,'' \emph{IEEE Transactions on Knowledge and Data Engineering}, 2021.

\bibitem{zhang2021elite}
H.~Zhang, L.~Cao, P.~VanNostrand, S.~Madden, and E.~A. Rundensteiner, ``Elite: Robust deep anomaly detection with meta gradient,'' in \emph{Proceedings of the 27th ACM SIGKDD Conference on Knowledge Discovery \& Data Mining}, 2021, pp. 2174--2182.

\bibitem{pang2019deep}
G.~Pang, C.~Shen, and A.~van~den Hengel, ``Deep anomaly detection with deviation networks,'' in \emph{Proceedings of the 25th ACM SIGKDD international conference on knowledge discovery \& data mining}, 2019, pp. 353--362.

\bibitem{jiang2023anomaly}
M.~Jiang, S.~Han, and H.~Huang, ``Anomaly detection with score distribution discrimination,'' \emph{arXiv preprint arXiv:2306.14403}, 2023.

\bibitem{perini2023learning}
L.~Perini, V.~Vercruyssen, and J.~Davis, ``Learning from positive and unlabeled multi-instance bags in anomaly detection,'' in \emph{Proceedings of the 29th ACM SIGKDD Conference on Knowledge Discovery and Data Mining}, 2023, pp. 1897--1906.

\bibitem{wei2019review}
Y.~Wei, Y.~Li, M.~Xu, and W.~Huang, ``A review of early fault diagnosis approaches and their applications in rotating machinery,'' \emph{Entropy}, vol.~21, no.~4, p. 409, 2019.

\bibitem{sha2022regional}
Y.~Sha, S.~Gou, J.~Faber, B.~Liu, W.~Li, S.~Schramm, H.~Stoecker, T.~Steckenreiter, D.~Vnucec, N.~Wetzstein \emph{et~al.}, ``Regional-local adversarially learned one-class classifier anomalous sound detection in global long-term space,'' in \emph{Proceedings of the 28th ACM SIGKDD Conference on Knowledge Discovery and Data Mining}, 2022, pp. 3858--3868.

\bibitem{ahmed2022rcad}
A.~H. Ahmed, M.~A. Riegler, S.~A. Hicks, and A.~Elmokashfi, ``Rcad: Real-time collaborative anomaly detection system for mobile broadband networks,'' in \emph{Proceedings of the 28th ACM SIGKDD Conference on Knowledge Discovery and Data Mining}, 2022, pp. 2682--2691.

\bibitem{himeur2021artificial}
Y.~Himeur, K.~Ghanem, A.~Alsalemi, F.~Bensaali, and A.~Amira, ``Artificial intelligence based anomaly detection of energy consumption in buildings: A review, current trends and new perspectives,'' \emph{Applied Energy}, vol. 287, p. 116601, 2021.

\bibitem{huang2021digital}
H.~Huang, L.~Yang, Y.~Wang, X.~Xu, and Y.~Lu, ``Digital twin-driven online anomaly detection for an automation system based on edge intelligence,'' \emph{Journal of Manufacturing Systems}, vol.~59, pp. 138--150, 2021.

\bibitem{degirmenci2022efficient}
A.~Degirmenci and O.~Karal, ``Efficient density and cluster based incremental outlier detection in data streams,'' \emph{Information Sciences}, vol. 607, pp. 901--920, 2022.

\bibitem{pang2023deep}
G.~Pang, C.~Shen, H.~Jin, and A.~van~den Hengel, ``Deep weakly-supervised anomaly detection,'' in \emph{Proceedings of the 29th ACM SIGKDD Conference on Knowledge Discovery and Data Mining}, 2023, pp. 1795--1807.

\bibitem{chang2023data}
C.-H. Chang, J.~Yoon, S.~{\"O}. Arik, M.~Udell, and T.~Pfister, ``Data-efficient and interpretable tabular anomaly detection,'' in \emph{Proceedings of the 29th ACM SIGKDD Conference on Knowledge Discovery and Data Mining}, 2023, pp. 190--201.

\bibitem{ju2020pumad}
H.~Ju, D.~Lee, J.~Hwang, J.~Namkung, and H.~Yu, ``Pumad: Pu metric learning for anomaly detection,'' \emph{Information Sciences}, vol. 523, pp. 167--183, 2020.

\bibitem{ma2015flexible}
Y.~Ma and W.~Yao, ``Flexible estimation of a semiparametric two-component mixture model with one parametric component,'' \emph{Electronic Journal of Statistics}, vol.~9, pp. 447--474, 2015.

\bibitem{al2017semiparametric}
D.~Al~Mohamad and A.~Boumahdaf, ``Semiparametric two-component mixture models when one component is defined through linear constraints,'' \emph{IEEE Transactions on Information Theory}, vol.~64, no.~2, pp. 795--830, 2017.

\bibitem{song2010estimating}
S.~Song, D.~L. Nicolae, and J.~Song, ``Estimating the mixing proportion in a semiparametric mixture model,'' \emph{Computational statistics \& data analysis}, vol.~54, no.~10, pp. 2276--2283, 2010.

\bibitem{shen2018mm}
Z.~Shen, M.~Levine, and Z.~Shang, ``An mm algorithm for estimation of a two component semiparametric density mixture with a known component,'' \emph{Electronic Journal of Statistics}, vol.~12, no.~1, pp. 1181--1209, 2018.

\bibitem{shokirov2010on}
B.~Shokirov, ``On a problem connected with mixture parameter estimation,'' \emph{Informacni Bulleten Ceske statisticke spolecnosti}, vol.~22, pp. 95--101, 2010.

\bibitem{patra2016estimation}
R.~K. Patra and B.~Sen, ``Estimation of a two-component mixture model with applications to multiple testing,'' \emph{Journal of the Royal Statistical Society: Series B (Statistical Methodology)}, vol.~78, no.~4, pp. 869--893, 2016.

\bibitem{bhattacharya2021fast}
A.~Bhattacharya, S.~Varambally, A.~Bagchi, and S.~Bedathur, ``Fast one-class classification using class boundary-preserving random projections,'' in \emph{Proceedings of the 27th ACM SIGKDD Conference on Knowledge Discovery \& Data Mining}, 2021, pp. 66--74.

\bibitem{robinson1972quadratically}
S.~M. Robinson, ``A quadratically-convergent algorithm for general nonlinear programming problems,'' \emph{Mathematical programming}, vol.~3, pp. 145--156, 1972.

\bibitem{mezHo2017generalization}
I.~Mez{\H{o}} and {\'A}.~Baricz, ``On the generalization of the lambert w function,'' \emph{Transactions of the American Mathematical Society}, vol. 369, no.~11, pp. 7917--7934, 2017.

\bibitem{otsu1979threshold}
N.~Otsu, ``A threshold selection method from gray-level histograms,'' \emph{IEEE transactions on systems, man, and cybernetics}, vol.~9, no.~1, pp. 62--66, 1979.

\bibitem{andrade2021harking}
C.~Andrade, ``Harking, cherry-picking, p-hacking, fishing expeditions, and data dredging and mining as questionable research practices,'' \emph{The Journal of clinical psychiatry}, vol.~82, no.~1, p. 25941, 2021.

\bibitem{keller2012hics}
F.~Keller, E.~Muller, and K.~Bohm, ``Hics: High contrast subspaces for density-based outlier ranking,'' in \emph{2012 IEEE 28th international conference on data engineering}.\hskip 1em plus 0.5em minus 0.4em\relax IEEE, 2012, pp. 1037--1048.

\bibitem{dua2019uci}
\BIBentryALTinterwordspacing
D.~Dua and C.~Graff, ``Uci machine learning repository,'' 2017. [Online]. Available: \url{http://archive.ics.uci.edu/ml}
\BIBentrySTDinterwordspacing

\bibitem{lazarevic2005feature}
A.~Lazarevic and V.~Kumar, ``Feature bagging for outlier detection,'' in \emph{Proceedings of the eleventh ACM SIGKDD international conference on Knowledge discovery in data mining}, 2005, pp. 157--166.

\bibitem{micenkova2015stamp}
B.~Micenkov{\'a}, J.~van Beusekom, and F.~Shafait, ``Stamp verification for automated document authentication,'' in \emph{Computational Forensics: 5th International Workshop, IWCF 2012, Tsukuba, Japan, November 11, 2012 and 6th International Workshop, IWCF 2014, Stockholm, Sweden, August 24, 2014, Revised Selected Papers}.\hskip 1em plus 0.5em minus 0.4em\relax Springer, 2015, pp. 117--129.

\bibitem{zhang2009new}
K.~Zhang, M.~Hutter, and H.~Jin, ``A new local distance-based outlier detection approach for scattered real-world data,'' in \emph{Advances in Knowledge Discovery and Data Mining: 13th Pacific-Asia Conference, PAKDD 2009 Bangkok, Thailand, April 27-30, 2009 Proceedings 13}.\hskip 1em plus 0.5em minus 0.4em\relax Springer, 2009, pp. 813--822.

\bibitem{quinlan1986induction}
J.~R. Quinlan, ``Induction of decision trees,'' \emph{Machine learning}, vol.~1, pp. 81--106, 1986.

\bibitem{ayres2000sisporto}
D.~Ayres-de Campos, J.~Bernardes, A.~Garrido, J.~Marques-de Sa, and L.~Pereira-Leite, ``Sisporto 2.0: a program for automated analysis of cardiotocograms,'' \emph{Journal of Maternal-Fetal Medicine}, vol.~9, no.~5, pp. 311--318, 2000.

\bibitem{frey1991letter}
P.~W. Frey and D.~J. Slate, ``Letter recognition using holland-style adaptive classifiers,'' \emph{Machine learning}, vol.~6, pp. 161--182, 1991.

\bibitem{kudo1999multidimensional}
M.~Kudo, J.~Toyama, and M.~Shimbo, ``Multidimensional curve classification using passing-through regions,'' \emph{Pattern Recognition Letters}, vol.~20, no. 11-13, pp. 1103--1111, 1999.

\bibitem{loh2011classification}
W.-Y. Loh, ``Classification and regression trees,'' \emph{Wiley interdisciplinary reviews: data mining and knowledge discovery}, vol.~1, no.~1, pp. 14--23, 2011.

\bibitem{campos2016evaluation}
G.~O. Campos, A.~Zimek, J.~Sander, R.~J. Campello, B.~Micenkov{\'a}, E.~Schubert, I.~Assent, and M.~E. Houle, ``On the evaluation of unsupervised outlier detection: measures, datasets, and an empirical study,'' \emph{Data mining and knowledge discovery}, vol.~30, pp. 891--927, 2016.

\bibitem{horton1996probabilistic}
P.~Horton and K.~Nakai, ``A probabilistic classification system for predicting the cellular localization sites of proteins.'' in \emph{Ismb}, vol.~4, 1996, pp. 109--115.

\bibitem{han2022adbench}
S.~Han, X.~Hu, H.~Huang, M.~Jiang, and Y.~Zhao, ``Adbench: Anomaly detection benchmark,'' \emph{Advances in Neural Information Processing Systems}, vol.~35, pp. 32\,142--32\,159, 2022.

\bibitem{li2020copod}
Z.~Li, Y.~Zhao, N.~Botta, C.~Ionescu, and X.~Hu, ``Copod: copula-based outlier detection,'' in \emph{2020 IEEE International Conference on Data Mining (ICDM)}.\hskip 1em plus 0.5em minus 0.4em\relax IEEE, 2020, pp. 1118--1123.

\bibitem{le1987modeles}
Y.~Le~Cun and F.~Fogelman-Souli{\'e}, ``Mod{\`e}les connexionnistes de l'apprentissage,'' \emph{Intellectica}, vol.~2, no.~1, pp. 114--143, 1987.

\bibitem{ruff2018deep}
L.~Ruff, R.~Vandermeulen, N.~Goernitz, L.~Deecke, S.~A. Siddiqui, A.~Binder, E.~M{\"u}ller, and M.~Kloft, ``Deep one-class classification,'' in \emph{International conference on machine learning}.\hskip 1em plus 0.5em minus 0.4em\relax PMLR, 2018, pp. 4393--4402.

\end{thebibliography}
	\appendix

    \section{Proofs}
    \subsection{Proof of Theorem~\ref{thm:p}}
    Let us restate Theorem~\ref{thm:p}.

    \addtocounter{theorem}{-2}

    \begin{theorem}\label{thm:pr}
	For $\lambda_1=0$ and $\lambda_2\neq 0$, the equation system~\eqref{eq:KKT} has the following solution for $p$
	\begin{equation}\tag{\ref{eq:p_mle_constrained}}
		p=\frac{1}{n-\Omega}\sum_{t=1}^{n}\B_t
	\end{equation}
	where $\Omega<n-\sum_{t=1}^{n}\B_t$ is the density surplus gradient, which is given by
	\begin{equation}\tag{\ref{eq:omega}}
		\Omega=
  \frac{\frac{\partial}{\partial \theta^{(1)}}\log\mathcal{L}}{\frac{\partial}{\partial \theta^{(1)}}\log I}I = \mathellipsis = \frac{\frac{\partial}{\partial \theta^{(q)}}\log\mathcal{L}}{\frac{\partial}{\partial \theta^{(q)}}\log I}I
	\end{equation}
 and where (either of the) derivatives of the right-hand side are evaluated at an MLE of $\theta$.
\end{theorem}
    
    \begin{proof}
	Since $\lambda_2\neq 0$ we can solve \eqref{eq:dt1} for $\lambda_2$ which gives
	\begin{equation}\label{eq:lambda_2}
		\lambda_2 = -\frac{\partial\mathcal{L}}{\partial \theta^{(1)}} \Big/ \left(2(1-p)\left(1-(1-p)I-\bar{P}\right)\frac{\partial}{\partial \theta^{(1)}} I\right).
	\end{equation}
	Plugging \eqref{eq:lambda_2} into \eqref{eq:dp} gives
	\begin{align*}
		\frac{\partial \mathcal{L}}{\partial p} + 2I\left(1-(1-p)I-\bar{P}\right)\left(\frac{\partial\mathcal{L}}{\partial \theta^{(1)}}\right)\Big/\\ \left(2(1-p)\left(1-(1-p)I-\bar{P}\right)\frac{\partial}{\partial \theta^{(1)}} I\right)=0
	\end{align*}
	since $\lambda_1=0$.
	This can be further simplified to 
	\begin{equation}\label{eq:intermediate_step_proof_p}
		\frac{\partial \mathcal{L}}{\partial p} + \frac{\partial\mathcal{L}}{\partial \theta^{(1)}}\frac{1}{1-p}\frac{I}{\frac{\partial}{\partial \theta^{(1)}}I} = 0.
	\end{equation}
	Now the rule $\frac{d}{dx}f(x) = f(x) \frac{d}{dx}\log f(x)$ and the shorthand notation $\log\mathcal{L}=\ell$ allows to combine \eqref{eq:omega}   and \eqref{eq:intermediate_step_proof_p} as
	\begin{equation*}
		\frac{\partial\ell}{\partial p} + \frac{1}{1-p}\frac{\partial \ell}{\partial \theta^{(1)}}\frac{1}{\frac{\partial}{\partial \theta^{(1)}} \log I} = \frac{\partial\ell}{\partial p} + \frac{\Omega}{1-p}=0.
	\end{equation*}
	Deriving the log likelihood with respect to $p$ gives
	\begin{equation*}
		\frac{\sum_{t=1}^{n}\B_t}{p} - \frac{n-\sum_{t=1}^{n}\B_t}{1-p} + \frac{\Omega}{1-p}=0
	\end{equation*}
	which can be solved for $p$ to obtain \eqref{eq:p_mle_constrained}.
	Note that $\theta^{(1)}$ was chosen arbitrarily, and that the same $p$ must be obtained for all $\theta^{(j)}$, $j\in 1,\mathellipsis,q$.
	Hence, \eqref{eq:omega} must also hold. 
\end{proof}
    \subsection{Proof of Theorem~\ref{thm:sigma2}}
    Let us restate Theorem~\ref{thm:sigma2}.
    \begin{theorem}\label{thm:sigma2r}
	For $\lambda_2\neq0$, Gaussian $f_N=\mathcal{N}(\mu,\sigma^2)$ and $R=[a,b]$, equation system~\eqref{eq:KKT} has the following solution for $\sigma^2$
	
	\begin{equation}\tag{\ref{eq:sigma2_solution}}
		\sigma^2=\frac{\displaystyle \beta m}{\displaystyle m+\beta W_r\left(\gamma\right)}
	\end{equation}
\end{theorem}
    \begin{proof}
	In the following we abbreviate $e_a=e^{-\frac{(a-\mu)^2}{2\sigma^2}}$, $e_b=e^{-\frac{(b-\mu)^2}{2\sigma^2}}$, $C_a=a-\mu$ and $C_b=b-\mu$.
	We start by stating the derivatives of $\ell$ and $I$ with respect to $\mu$ and $\sigma$:
	\begin{equation}\label{eq:derivatives}
		\begin{aligned}
			\frac{\partial \ell}{\partial \mu}&=\sum_{t=1}^{n}(1-\B_t) \frac{x_t-\mu}{\sigma^2}\\
			\frac{\partial \ell}{\partial \sigma}&=\sum_{t=1}^{n}(1-\B_t) \frac{(x_t-\mu)^2-\sigma^2}{\sigma^3}\\
			\frac{\partial I}{\partial \mu} &=\int_{a}^{b}\frac{\partial}{\partial \mu} \frac{1}{\sqrt{2\pi}\sigma}e^{-\frac{(x-\mu)^2}{2\sigma^2}}= \frac{1}{\sqrt{2\pi}\sigma}\left(e_a-e_b\right)\\
			\frac{\partial I}{\partial \sigma} &=\int_{a}^{b}\frac{\partial}{\partial \sigma} \frac{1}{\sqrt{2\pi}\sigma}e^{-\frac{(x-\mu)^2}{2\sigma^2}}= \frac{1}{\sqrt{2\pi}\sigma^2}\left(C_a e_a-C_b e_b\right)		
		\end{aligned}
	\end{equation}	
	Plugging \eqref{eq:derivatives} into \eqref{eq:omega} and dividing both sides by $I$ gives
	\begin{align*}
		&\sum_{t=1}^{n}(1-\B_t) \frac{x_t-\mu}{\sigma^2} \Big / \left(\frac{1}{\sqrt{2\pi}\sigma}\left(e_a-e_b\right)\right) \\&= \sum_{t=1}^{n}(1-\B_t) \frac{(x_t-\mu)^2-\sigma^2}{\sigma^3} \Big / \left(\frac{1}{\sqrt{2\pi}\sigma^2}\left(C_a e_a-C_b e_b\right)	\right).
	\end{align*}
	By further abbreviating $s_B = \sum_{t=1}^n \B_t$, $s_x = \sum_{t=1}^n (1-\B_t)x_t$, $s_{x^2}=\sum_{t=1}^{n}(1-\B_t)x_t^2$ and eliminating duplicate terms, this statement can be simplified to
	\begin{align}
		\frac{s_x-(n-s_B)\mu}{e_a-e_b} = \frac{s_{x^2}-2\mu s_x + (n-s_B)(\mu^2-\sigma^2)}{C_a e_a-C_b e_b}\nonumber\\
		\frac{C_a e_a-C_b e_b}{e_a-e_b} = \frac{s_{x^2}-2\mu s_x + (n-s_B)(\mu^2-\sigma^2)}{s_x-(n-s_B)\mu}\label{eq:two_fracs}.
	\end{align}
	The left-hand side of \eqref{eq:two_fracs} can be rewritten as
	\begin{equation*}
		\frac{C_a e_a-C_b e_b}{e_a-e_b} = \frac{a e_a - b e_b}{e_a - e_b} -  \mu = \mathcal{E}-\mu.
	\end{equation*}
	Inserting this into \eqref{eq:two_fracs} and multiplying with the right-hand side denominator gives
	\begin{equation}\label{eq:last_before_bar}
		(\mathcal{E}-\mu)\left(s_x-(n-s_B)\mu\right) = s_{x^2}-2\mu s_x + (n-s_B)(\mu^2-\sigma^2).
	\end{equation}
	If we study \eqref{eq:last_before_bar} closely, it becomes apparent that dividing both sides by $(n-s_B)$ converts several terms to the averages defined in \eqref{eq:x_bar} and \eqref{eq:x2_bar}
	\begin{align}
		(\mathcal{E}-\mu)(\bar{x}-\mu)=\bar{x^2}-2\mu\bar{x}+\mu^2-\sigma^2\nonumber\\
		\sigma^2 = \bar{x^2} - \mu \bar{x} + (\mu - \bar{x})\mathcal{E}.\label{eq:sigma2_simple}
	\end{align}
	Note that \eqref{eq:sigma2_simple} coincides with the unconstrained Gaussian MLE if $\mu=\bar{x}$ or $\mathcal{E}=0$.
	While \eqref{eq:sigma2_simple} is easy to understand, it is not a solution for $\sigma^2$ since the term $\mathcal{E}$ contains $\sigma^2$.
	In the following steps, we rearrange \eqref{eq:sigma2_simple} to the standard form for the $r$-Lampert function given in \cite{mezHo2017generalization}.
	
	We simplify $\mathcal{E}$ using \eqref{eq:m}:
	\begin{equation}\label{eq:e_term}
		\mathcal{E} = \frac{a e_a - b e_b}{e_a - e_b} = b + \frac{(a-b)e_a}{e_a-e_b} = b+ \frac{a-b}{1-e_b /e_a} = b + \frac{a-b}{1-e^{\frac{m}{\sigma^2}}}
	\end{equation}
	Now, let $z=\frac{1}{\sigma^2}$ and plugging \eqref{eq:beta} and \eqref{eq:e_term} into \eqref{eq:sigma2_simple} gives
	\begin{equation*}
		\frac{1}{z} = \beta + (\mu-\bar{x})\frac{a-b}{1-e^{mz}}.
	\end{equation*}
	Multiplying with the denominators and plugging in \eqref{eq:alpha} gives
	\begin{align}
		1-e^{mz} &= z\left(\alpha-\beta e^{mz}\right)\nonumber\\
		1-\alpha z &= e^{mz}\left(1-\beta z\right)\nonumber\\
		\frac{\alpha}{\beta}\left(z-\frac{1}{\alpha}\right)&=e^{mz}\left(z-\frac{1}{\beta}\right)\nonumber\\
        \frac{\alpha}{\beta}&=e^{mz}\left(z-\frac{1}{\beta}\right) \Big/ \left(z-\frac{1}{\alpha}\right)\label{eq:almost_there}
	\end{align}
	Equation~\eqref{eq:almost_there} matches exactly the form of \cite[Theorem 3]{mezHo2017generalization} and thus we have
	\begin{equation*}
		z = \frac{1}{\beta}+\frac{1}{m}W_{-\frac{\alpha}{\beta}e^{-\frac{m}{\beta}}}\left(m\left(\frac{\alpha-\beta}{\beta^2}\right)e^{-\frac{m}{\beta}}\right).
	\end{equation*}
	Lastly, we replace $z$ by $\frac{1}{\sigma^2}$, and multiply with $\beta m$, and the proof is complete.
\end{proof}


\section{Solution for the Uniform Distribution}
In the main matter, we focused our theoretical exposition of the estimation of $f_N$'s parameters on the Gaussian distribution.
Here, we describe the (somewhat disappointing) solution for uniform $f_N$.
\begin{proposition}
For $\lambda_1\neq0,\lambda_2\neq0$, Uniform $f_N=\mathcal{U}(k,l)$ and $R=[a,b]$, equation system~\eqref{eq:KKT} has no constrained solution. The constraint is always fulfilled.
\end{proposition}
\begin{proof}
    We conduct a proof by contradiction.
    Assume that $0<p<1,k,l$ s.t. $\Omega_k=\Omega_l$ as per Thm~\ref{thm:p}.
    Uniform $f_N$ is given by \begin{equation*}
    f_N(x,k,l)=
        \begin{cases}
            \frac{1}{l-k} & k\leq x\leq l\\
            0 & else
        \end{cases}
    \end{equation*}
    Assume without loss of generality that $k\leq x\leq l$.
    The quantities in $\Omega_l$ are given by \begin{align*}
        \frac{\partial\ell}{\partial l} &= \frac{n-s_B}{l-k}\\
        I &= \frac{b-a}{l-k}\\
        \frac{\partial I}{\partial l} &= -\frac{b-a}{(l-k)^2}
    \end{align*}
    and hence, we have \begin{equation*}
        \Omega_l=\frac{\frac{\partial\ell}{\partial l}}{\frac{\partial I}{l}}I=s_B-n.
    \end{equation*}
    Plugging this into \eqref{eq:p_mle_constrained} gives
    \begin{equation*}
    p = \frac{s_B}{n-\Omega_l}=\frac{s_B}{n-n+s_B}=1
    \end{equation*}
    which violates the assumption and completes the desired contradiction.
\end{proof}
\section{Results on development datasets}
For the sake of completeness, we report the results of all methods studied in RQ2.1 on the development datasets.
We did not report these results in the main matter since we used these datasets to test various setups of our method.
However, we believe that the results might nevertheless be interesting to some readers.
The overall ranking of methods on the development datasets is slightly, but not considerable different from the ranking on the evaluation datasets.
The most notable difference is the improved performance of COPOD.
This was expected, since the development datasets include several datasets that were studied in the original COPOD publication~\cite{li2020copod}.

\begin{table*}[!htb]
    \caption{Results for RQ2 on the development dataset. Depicted values in rows starting with dataset names are AUC-ROC scores. We did not report these results in the main matter since we used these datasets to test various setups of our method.
The overall ranking of methods on the development datasets is slightly, but not considerable different from the ranking on the evaluation datasets.
The most notable difference is the improved performance of COPOD.
This was expected, since the development datasets include several datasets that were studied in the original COPOD publication~\cite{li2020copod}.}
    \label{tab:dev}
    \centering
    
    \begin{tabular}{l|ccc|ccc|cc}
    Method class&\multicolumn{3}{c|}{Heuristical} & \multicolumn{3}{c|}{Probabilistic} & \multicolumn{2}{c}{Deep learning}\\
    \textbf{Data / Measure}& \textbf{LOF} & \textbf{IForest }& \textbf{COPOD} & \textbf{EM} & \textbf{MLE }& \textbf{CAMLE}& \textbf{AutoEnc }& \textbf{DeepSVDD}\\
    \toprule
  Glass & 0.64 & 0.79 & 0.76 & 0.75 & 0.72 & 0.76 & 0.73 & 0.68 \\ 
  Heart & 0.55 & 0.60 & 0.65 & 0.73 & 0.73 & 0.74 & 0.44 & 0.53 \\ 
  Hep & 0.49 & 0.69 & 0.80 & 0.37 & 0.36 & 0.40 & 0.76 & 0.56 \\ 
  Ion & 0.90 & 0.85 & 0.79 & 0.77 & 0.76 & 0.77 & 0.83 & 0.61 \\ 
  Lymph & 1.00 & 1.00 & 1.00 & 1.00 & 1.00 & 0.98 & 1.00 & 0.83 \\ 
  Park & 0.47 & 0.50 & 0.51 & 0.48 & 0.42 & 0.42 & 0.25 & 0.61 \\ 
  Pima & 0.59 & 0.67 & 0.65 & 0.59 & 0.56 & 0.56 & 0.63 & 0.46 \\ 
  Stamps& 0.51 & 0.90 & 0.93 & 0.93 & 0.92 & 0.94 & 0.90 & 0.55 \\ 
  Shuttle& 0.64 & 0.87 & 0.82 & 0.90 & 0.91 & 0.74 & 0.94 & 0.67 \\ 
  WPBC & 0.46 & 0.50 & 0.52 & 0.49 & 0.48 & 0.51 & 0.44 & 0.51 \\ 
  \midrule
  ø Rank (per method class) & 2.8 & 1.6 & 1.5 & 1.6 & 2.5 & 1.8 & 1.3 & 1.7\\
  \midrule
  ø Rank (overall) &5.9 &    3.0   &   \multicolumn{1}{c}{2.5}  &    4.0   &   5.2   &   \multicolumn{1}{c}{4.5}   &   4.6  &    5.9 \\
  Total runtime (seconds)  &0.01   &  0.74   &   \multicolumn{1}{c}{0.00}   &   0.01   &   0.02   &   \multicolumn{1}{c}{0.46}  &   12.85  &  2.84 
    \end{tabular}
    
    \end{table*}

\section{Sensitivity study on all other evaluation datasets}
For the sake of completeness, we conduct the same experiment as in RQ2.2 for all other evaluation datasets.
The results of this experiment are depicted in Figure~\ref{fig:all_sensitive}.
In this figure, the dashed blue line indicates CAMLE's performance on this dataset using the default parametrization as reported in Table~\ref{tab:benchmar_results}.
These results indicate high stability of the method with respect to the estimated AFR.
Please note the small variance of the results.

\begin{figure*}[!htb]
    \centering

    \vspace{2cm}
    
    \textbf{Cardio} \hspace{2cm} \textbf{Cardiotocography} \hspace{2.5cm}\textbf{Letter}
    
    \includegraphics[scale=0.4]{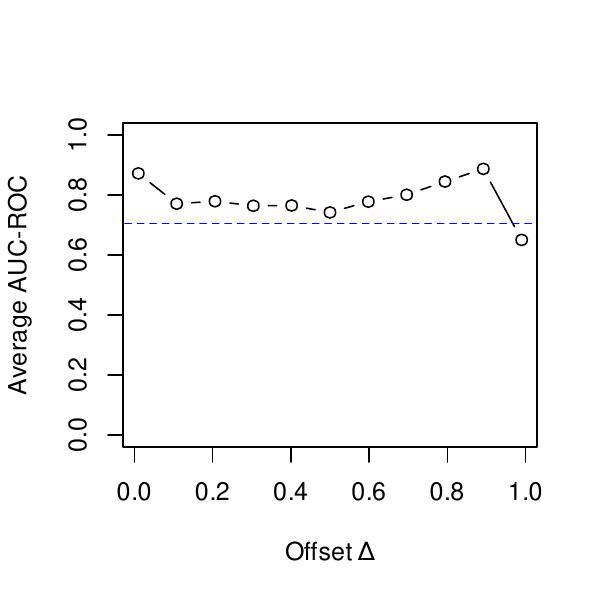}
    \includegraphics[scale=0.4]{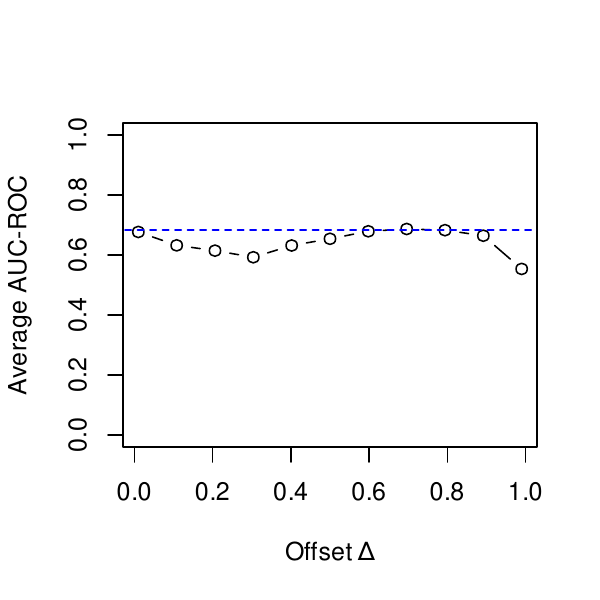}
    \includegraphics[scale=0.4]{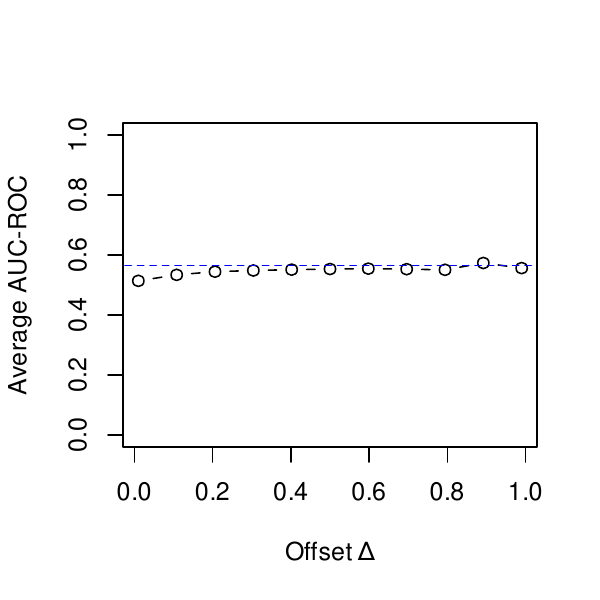}

    \vspace{1cm}

    \textbf{SatImage} \hspace{2.5cm}\textbf{Vowels}\hspace{2.5cm} \textbf{Waveform}
    
    \includegraphics[scale=0.4]{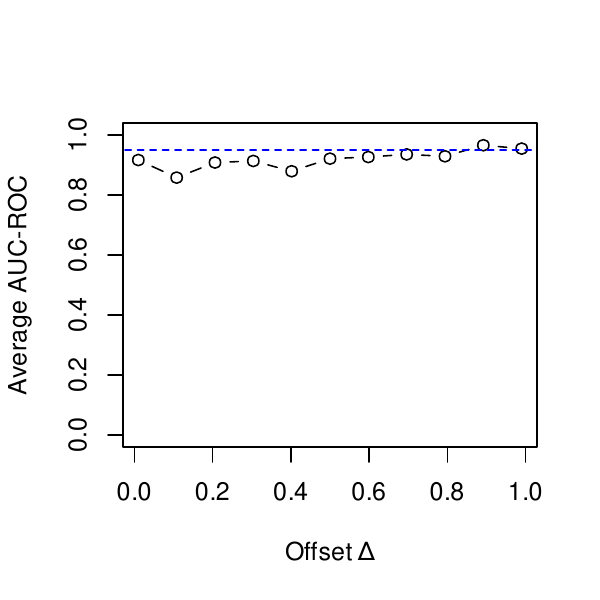}
    \includegraphics[scale=0.4]{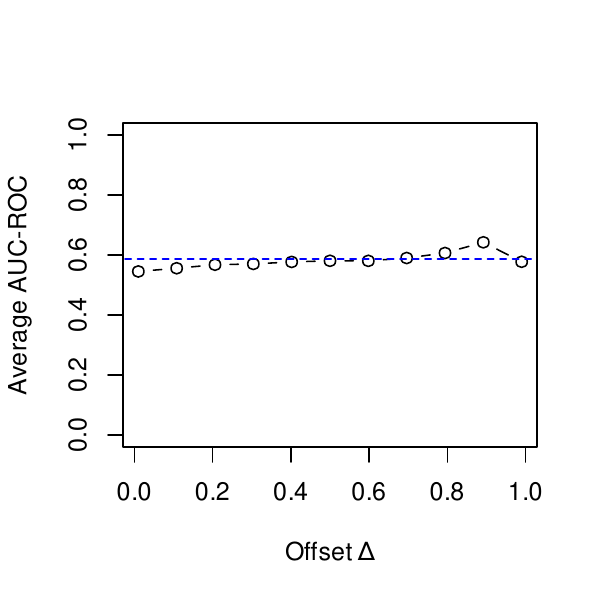}
    \includegraphics[scale=0.4]{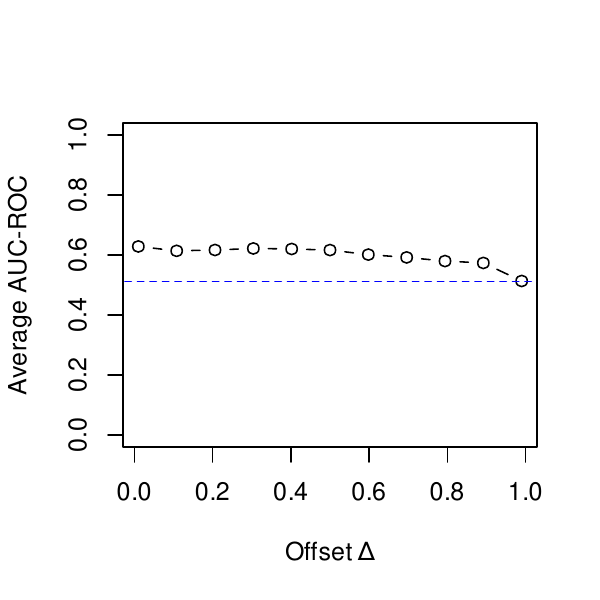}

    \vspace{1cm}
    
    \textbf{Wilt}\hspace{3cm} \textbf{Yeast}
    
    \includegraphics[scale=0.4]{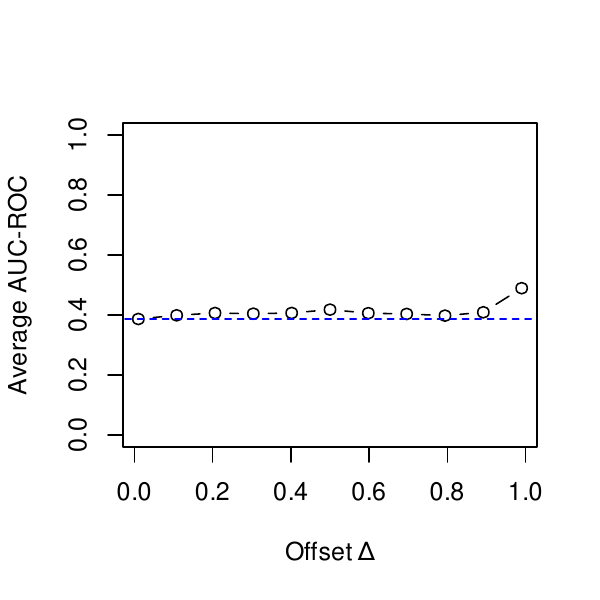}
    \includegraphics[scale=0.4]{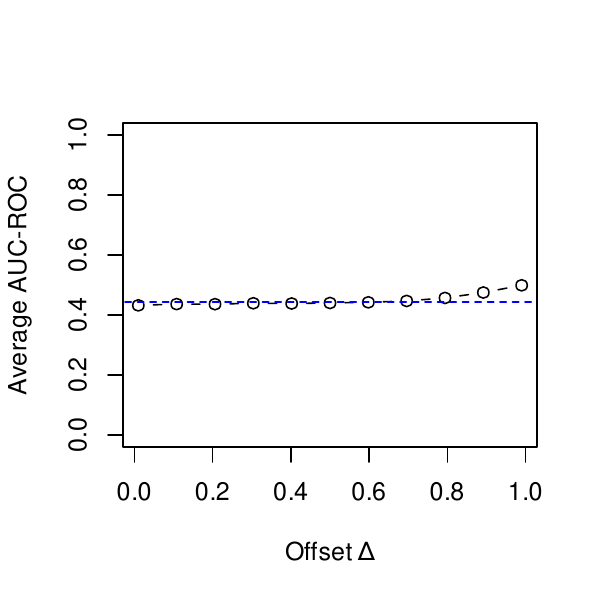}
    \caption{Sensitivity analysis for all other evaluation datasets. The dashed blue line indicates CAMLE's performance on this dataset using the default parametrization as reported in Table~\ref{tab:benchmar_results}. These results indicate high stability of the method with respect to the estimated AFR.
    Please note the small variance of the results.}
    \label{fig:all_sensitive}
\end{figure*}

\section{Visualization of Office Dataset}
A visualization of the newly collected Office dataset is depicted in Figure~\ref{fig:chol}. An anonymous worker's daily work time deviation has natural fluctuation due to flexible work time. Some deviations are anomalous, e.g., because the worker forget to log out when leaving the company.
The region $[-29;29]$ contains no anomalies and is a valid AFR derived from domain knowledge.
\begin{figure*}[!htb]
	\centering
	\includegraphics[scale=0.35]{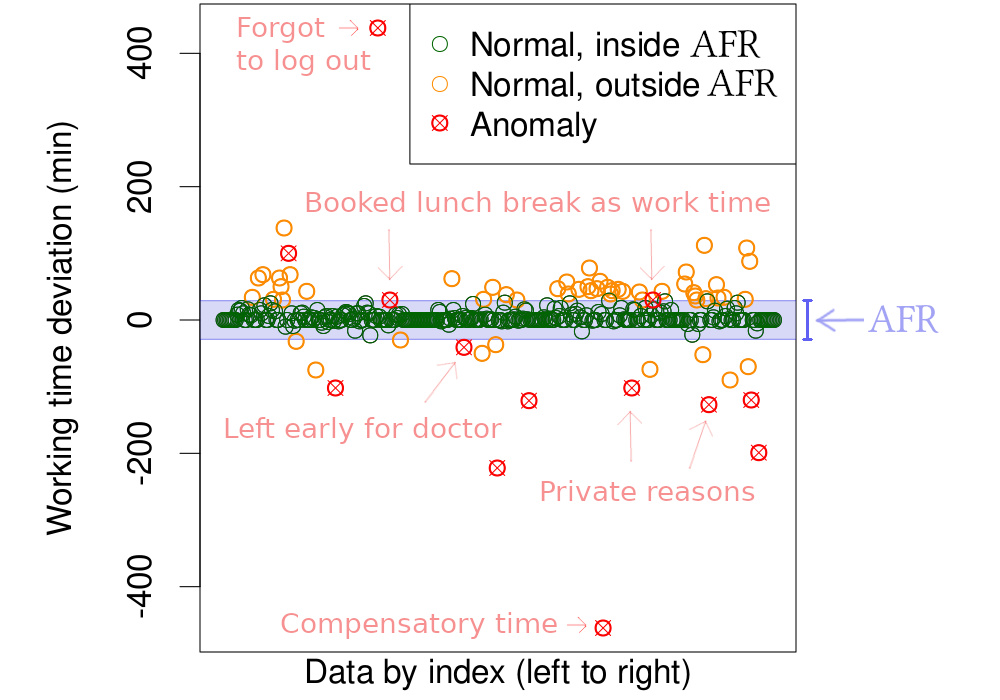}
	\caption{Visualization of the \textsc{Office} dataset. An anonymous worker's daily work time deviation has natural fluctuation due to flexible work time. Some deviations are anomalous, e.g., because the worker forget to log out when leaving the company.
	The region $[-29;29]$ contains no anomalies and is a valid AFR derived from domain knowledge.}
	\label{fig:chol}
\end{figure*}

\end{document}